\def\comments{1}

\documentclass[11pt]{article}
\usepackage{ifthen}
\usepackage{mdwlist}
\usepackage{amsmath,amssymb,amsfonts,amsthm}
\usepackage{bbm}
\usepackage[colorlinks=true,linkcolor=blue,citecolor=blue,urlcolor=blue]{hyperref}
\usepackage{enumitem}
\usepackage{graphicx}
\usepackage{xspace}
\usepackage{verbatim}
\usepackage{algpseudocode}
\usepackage[boxruled,vlined]{algorithm2e}
	\SetKwFor{While}{While}{}{}
	\SetKwFor{For}{For}{}{}
	\SetKwIF{If}{ElseIf}{Else}{If}{}{elif}{Else}{}
	\SetKw{Return}{Return}
\usepackage[margin=1.2in]{geometry}
\usepackage{thm-restate}
\usepackage{latexsym}
\usepackage{epsfig}
\usepackage[usenames,dvipsnames]{xcolor}
\usepackage{microtype}
\def\eps{\ve}
\renewcommand{\epsilon}{\ve}
\def\ve{\varepsilon}

\newcommand{\id}{\mathbb{I}}

\newcommand{\ex}[2]{{\ifx&#1& \mathbb{E} \else \underset{#1}{\mathbb{E}} \fi \left(#2\right)}}
\newcommand{\pr}[2]{{\ifx&#1& \mathbb{P} \else \underset{#1}{\mathbb{P}} \fi \left(#2\right)}}
\newcommand{\var}[2]{{\ifx&#1& \mathrm{Var} \else \underset{#1}{\mathrm{Var}} \fi \left(#2\right)}}

\newcommand{\N}{\mathbb{N}}
\newcommand{\R}{\mathbb{R}}

\newcommand{\cM}{\mathcal{M}}
\newcommand{\cN}{\mathcal{N}}
\newcommand{\cP}{\mathcal{P}}
\newcommand{\cR}{\mathcal{R}}
\newcommand{\cX}{\mathcal{X}}

\newcommand{\pmo}{\{\pm 1\}}

\newtheorem{theorem}{Theorem}

\newtheorem{remark}[theorem]{Remark}

\newtheorem{lemma}[theorem]{Lemma}
\newtheorem{claim}[theorem]{Claim}

\newtheorem{definition}[theorem]{Definition}

\newcommand{\ignore}[1]{}
\numberwithin{theorem}{section} 
\numberwithin{nontheorem}{section} 
\numberwithin{proposition}{section} 
\numberwithin{observation}{section} 
\numberwithin{remark}{section} 
\numberwithin{fact}{section} 
\numberwithin{lemma}{section} 
\numberwithin{claim}{section} 
\numberwithin{corollary}{section} 
\numberwithin{case}{section} 
\numberwithin{dfn}{section} 
\numberwithin{definition}{section} 
\numberwithin{question}{section} 
\numberwithin{openquestion}{section} 
\numberwithin{res}{section}

\ifnum\comments=1
\newcommand{\mynote}[2]{{\marginpar{\color{#1} \tiny \sf #2}}}
\else
\newcommand{\mynote}[2]{}
\fi
\newcommand{\jnote}[1]{\mynote{blue}{J: {#1}}}
\newcommand{\gnote}[1]{\mynote{red}{G: {#1}}}

\let\originalleft\left
\let\originalright\right
\renewcommand{\left}{\mathopen{}\mathclose\bgroup\originalleft}
\renewcommand{\right}{\aftergroup\egroup\originalright}

\newcommand{\from}{\colon}

\title{A Primer on Private Statistics}
\author {
Gautam Kamath\thanks{Cheriton School of Computer Science, University of Waterloo. {\tt g@csail.mit.edu}. Supported by a University of Waterloo startup grant.}
\and
Jonathan Ullman\thanks{Khoury College of Computer Sciences, Northeastern University. {\tt jullman@ccs.neu.edu}.  Supported by NSF grants CCF-1750640, CNS-1816028, and CNS-1916020.}
}

\begin{document}
\maketitle

\begin{abstract}
  Differentially private statistical estimation has seen a flurry of developments over the last several years. 
  Study has been divided into two schools of thought, focusing on empirical statistics versus population statistics. 
  We suggest that these two lines of work are more similar than different by giving examples of methods that were initially framed for empirical statistics, but can be applied just as well to population statistics.
  We also provide a thorough coverage of recent work in this area.
\end{abstract}

\section{Introduction}
Statistics and machine learning are now ubiquitous in data analysis. Given a dataset, one immediately wonders what it allows us to infer about the underlying population. However, modern datasets don't exist in a vacuum: they often contain sensitive information about the individuals they represent. Without proper care, statistical procedures will result in gross violations of privacy. Motivated by the shortcomings of ad hoc methods for data anonymization, Dwork, McSherry, Nissim, and Smith introduced the celebrated notion of differential privacy~\cite{DworkMNS06}.

From its inception, some of the driving motivations for differential privacy were applications in statistics and the social sciences, notably disclosure limitation for the US Census.  And yet, the lion's share of differential privacy research has taken place within the computer science community.  As a result, the specific applications being studied are often not formulated using statistical terminology, or even as statistical problems.  Perhaps most significantly, much of the early work in computer science (though definitely not all) focus on estimating some property \emph{of a dataset} rather than estimating some property \emph{of an underlying population}.

Although the earliest works exploring the interaction between differential privacy and classical statistics go back to at least 2009~\cite{VuS09,FienbergRY10}, the emphasis on differentially private statistical inference in the computer science literature is somewhat more recent.  However, while earlier results on differential privacy did not always formulate problems in a statistical language, statistical inference was a key motivation for most of this work.  As a result many of the techniques that were developed have direct applications in statistics, for example establishing minimax rates for estimation problems.  

The purpose of this series of blog posts is to highlight some of those results in the computer science literature, and present them in a more statistical language.  Specifically, we will discuss:
\begin{itemize}
	\item Tight minimax lower bounds for privately estimating the mean of a multivariate distribution over $\R^d$, using the technique of \emph{tracing attacks} developed in~\cite{BunUV14, SteinkeU15j, DworkSSUV15, BunSU17, SteinkeU17, KamathLSU19}.

	\item Upper bounds for estimating a distribution in Kolmogorov distance, using the ubiquitous \emph{binary-tree mechanism} introduced in~\cite{DworkNPR10,ChanSS11}.
\end{itemize}

In particular, we hope to encourage computer scientists working on differential privacy to pay more attention to the applications of their methods in statistics, and share with statisticians many of the powerful techniques that have been developed in the computer science literature.

\subsection{Formulating Private Statistical Inference} 

Essentially every differentially private statistical estimation task can be phrased using the following setup.
We are given a dataset $X = (X_1, \dots, X_n)$ of size $n$, and we wish to design an algorithm $M \in \cM$ where $\cM$ is the class of mechanisms that are both:
\begin{enumerate}
  \item differentially private, and 
  \item accurate, either in expectation or with high probability, according to some task-specific measure.
\end{enumerate}
A few comments about this framework are in order.  First, although the accuracy requirement is stochastic in nature (i.e., an algorithm might not be accurate depending on the randomness of the algorithm and the data generation process), the privacy requirement is worst-case in nature.  That is, the algorithm must protect privacy for every dataset $X$, even those we believe are very unlikely.  

Second, the accuracy requirement is stated rather vaguely.  This is because the notion of accuracy of an algorithm is slightly more nuanced, depending on whether we are concerned with \emph{empirical} or \emph{population} statistics.  A particular emphasis of these blog posts is to explore the difference (or, as we will see, the lack of a difference) between these two notions of accuracy.
The former estimates a quantity of the observed dataset, while the latter estimates a quantity of an unobserved distribution which is assumed to have generated the dataset. 

More precisely, the former can be phrased in terms of empirical loss, of the form:
\[
\min_{M \in \cM}~\max_{X \in \cX}~\ex{M}{\ell(M(X), f(X))},
\]
where $\cM$ is some class of \emph{randomized estimators} (e.g.,~differentially private estimators), $\cX$ is some class of \emph{datasets}, $f$ is some quantity of interest, and $\ell$ is some \emph{loss function}.  That is, we're looking to find an estimator that has small expected loss on \emph{any dataset} in some class.

In contrast, statistical minimax theory looks at statements about population loss, of the form:
\[
\min_{M \in \cM}~\max_{P \in \cP}~\ex{X \sim P, M}{\ell(M(X),f(P))},
\]
where $\cP$ is some family of \emph{distributions} over datasets (typically consisting of i.i.d.\ samples).  That is, we're looking to find an estimator that has small expected loss on random data from \emph{any distribution} in some class.  In particular, note that the randomness in this objective additionally includes the data generating procedure $X \sim P$.

These two formulations are formally very different in several ways.  First, the empirical formulation requires an estimator to have small loss on \emph{worst-case} datasets, whereas the statistical formulation only requires the estimator to have small loss \emph{on average} over datasets drawn from certain distributions.  Second, the statistical formulation requires that we estimate the unknown quantity $f(P)$, and thus necessitates a solution to the non-private estimation problem.  On the other hand, the empirical formulation only asks us to estimate the known quantity $f(X)$, and thus if there were no privacy constraint it would always be possible to compute $f(X)$ exactly.  Third, typically in the statistical formulation, we require that the dataset is drawn i.i.d., which means that we are more constrained when proving lower bounds for estimation than we are in the empirical problem.

However, in practice,\footnote{More precisely, in the practice of doing theoretical research.} these two formulations are more alike than they are different, and results about one formulation often imply results about the other formulation.  On the algorithmic side, classical statistical results will often tell us that $\ell(f(X),f(P))$ is small, in which case algorithms that guarantee $\ell(M(X),f(X))$ is small also guarantee $\ell(M(X),f(P))$ is small. 

Moreover, typical lower bound arguments for empirical quantities are often statistical in nature.  These typically involving constructing some simple ``hard distribution'' over datasets such that no private algorithm can estimate well on average for this distribution, and thus these lower bound arguments also apply to estimating population statistics for some simple family of distributions.

We will proceed to give some examples of estimation problems that were originally studied by computer scientists with the empirical formulation in mind.  These results either implicitly or explicitly provide solutions to the corresponding population versions of the same problems---our goal is to spell out and illustrate these connections.

\section{Differential Privacy Background}

Let $X = (X_1,X_2,\dots,X_n) \in \cX^n$ be a collection of $n$ samples where each individual sample comes from the domain $\cX$.  We say that two samples $X,X' \in \cX^*$ are \emph{adjacent}, denoted $X \sim X'$, if they differ on at most one individual sample.  Intuitively, a randomized algorithm $M$, which is often called a \emph{mechanism} for historical reasons, is \emph{differentially private} if the distribution of $M(X)$ and $M(X')$ are similar for every pair of adjacent samples $X,X'$.

\begin{definition}[\cite{DworkMNS06}]
	A mechanism $M \from \cX^n \to \cR$ is \emph{$(\eps,\delta)$-differentially private} if for every pair of adjacent datasets $X \sim X'$, and every (measurable) $R \subseteq R$
	$$
	\pr{}{M(X) \in R} \leq e^{\eps} \cdot \pr{}{M(X') \in R} + \delta.
	$$
\end{definition}

We let $\cM_{\eps,\delta}$ denote the set of mechanisms that satisfy $(\eps,\delta)$-differential privacy.
\begin{remark}
	To simplify notation, and to maintain consistency with the literature, we adopt the convention of defining the mechanism only for a fixed sample size $n$.  What this means in practice is that the mechanisms we describe treat the sample size $n$ is \emph{public information} that need not be kept private.  While one could define a more general model where $n$ is not fixed, it wouldn't add anything to this discussion other than additional complexity.
\end{remark}

\begin{remark}
	In these blog posts, we stick to the most general formulation of differential privacy, so-called \emph{approximate differential privacy}, i.e. $(\eps,\delta)$-differential privacy for $\delta > 0$ essentially because this is the notion that captures the widest variety of private mechanisms.  Almost all of what follows would apply equally well, with minor technical modifications, to slightly stricter notions of \emph{concentrated differential privacy}~\cite{DworkR16,BunS16,BunDRS18}, \emph{R\'{e}nyi differential privacy}~\cite{Mironov17}, or \emph{Gaussian differential privacy}~\cite{DongRS19}.  While so-called \emph{pure differential privacy}, i.e. $(\eps,0)$-differential privacy has also been studied extensively, this notion is artificially restrictive and excludes many differentially private mechanisms.
\end{remark}

A key property of differential privacy that helps when desinging efficient estimators is \emph{closure under postprocessing}:
\begin{lemma}[Post-Processing~\cite{DworkMNS06}] \label{lem:post-processing}
	If $M \from \cX^n \to \cR$ is $(\eps,\delta)$-differentially private and $M' \from \cR \to \cR'$ is any randomized algorithm, then $M' \circ M$ is $(\eps,\delta)$-differentially private.
\end{lemma}

The estimators we present in this work will use only one tool for achieving differential privacy, the \emph{Gaussian Mechanism}.
\begin{lemma}[Gaussian Mechanism] \label{lem:gauss-mech}
	Let $f \from \cX^n \to \R^d$ be a function and let 
	$$
	\Delta_{f} = \sup_{X\sim X'} \| f(X) - f(X') \|_2
	$$
	denote its \emph{$\ell_2$-sensitivity}.  The \emph{Gaussian mechanism}
	$$
	M(X) = f(X) + \cN\left(0 , \frac{2 \log(2/\delta)}{\eps^2} \cdot \Delta_{f}^2 \cdot \id_{d \times d} \right)
	$$
	satisfies $(\eps,\delta)$-differential privacy.
\end{lemma}
   
\section{Mean Estimation in $\R^d$}

Let's take a dive into the problem of \emph{private mean estimation} for some family $\cP$ of multivariate distributions over $\R^d$.  This problem has been studied for various families $\cP$ and various choices of loss function.  Here we focus on perhaps the simplest variant of the problem, in which $\cP$ contains distributions of bounded support $[\pm 1]^d$ and the loss is the $\ell_2^2$ error.  We emphasize, however, that the methods we discuss here are quite versatile and can be used to derive minimax bounds for other variants of the mean-estimation problem.

Note that, by a simple argument, the non-private minimax rate for this class is achieved by the empirical mean, and is 
\begin{equation}
\max_{P \in \cP} \ex{X_{1 \cdots n} \sim P}{\| \overline{X} - \mu\|_2^2} = \frac{d}{n}.
\end{equation}
The main goal of this section is to derive the minimax bound
\begin{equation} \label{eq:Rd-minimax}
\min_{M \in \cM_{\eps,\frac{1}{n}}} \max_{P \in \cP} \ex{X_{1 \cdots n} \sim P}{\| M(X_{1 \cdots n}) - \mu \|_2^2} = \frac{d}{n} + \tilde\Theta\left(\frac{d^2}{\eps^2 n^2}\right).\footnote{$\tilde \Theta(f(n))$ is a slight abuse of notation -- it refers to a function which is both $O(f(n) \log^{c_1} f(n))$ and $\Omega(f(n) \log^{c_2} f(n))$ for some constants $c_1, c_2$.}
\end{equation}

The proof of this lower bound is based on \emph{robust tracing attacks}, also called \emph{membership inference attacks}, which were developed in a chain of papers~\cite{BunUV14, SteinkeU15j, DworkSSUV15,BunSU17,SteinkeU17,KamathLSU19}.  We remark that this lower bound is almost identical to the minimax bound for mean estimation proven in the much more recent work of Cai, Wang, and Zhang~\cite{CaiWZ19}, but it lacks tight dependence on the parameter $\delta$, which we discuss in the following remark.

\begin{remark}
The choice of $\delta = 1/n$ in \eqref{eq:Rd-minimax} may look strange at first.  For the upper bound this choice is arbitrary---as we will see, we can upper bound the rate for any $\delta > 0$ at a cost of a factor of $O(\log(1/\delta))$.  The lower bound applies only when $\delta \leq 1/n$.  Note that the rate is qualitatively different when $\delta \gg 1/n$.  However, we emphasize that $(\eps,\delta)$-differential privacy is not a meaningful privacy notion unless $\delta \ll 1/n$.  In particular, the mechanism that randomly outputs $\delta n$ elements of the sample satisfies $(0,\delta)$-differential privacy.  However, when $\delta \gg 1/n$, this mechanism completely violates the privacy of $\gg 1$ person in the dataset.  Moreover, taking the empirical mean of these $\delta n$ samples gives rate $d/\delta n$, which would violate our lower bound when $\delta$ is large enough.  On the other hand, we would expect the minimax rate to become slower when $\delta \ll 1/n$.  This expectation is, in fact, correct, however the proof we present does not give the tight dependence on the parameter $\delta$.  See~\cite{SteinkeU15j} for a refinement that can obtain the right dependence on $\delta$, and~\cite{CaiWZ19} for the details of how to apply this refinement in the i.i.d.\ setting.
\end{remark}

\subsection{A Simple Upper Bound}

\begin{theorem}
For every $n \in \N$, and every $\eps,\delta > 0$, there exists an $(\eps,\delta)$-differentially private private mechanism $M$ such that
\begin{equation} \label{eq:mean-est-ub}
\max_{P \in \cP} \ex{X_{1 \cdots n} \sim P}{\| M(X_{1 \cdots n}) - \mu \|_2^2} \leq \frac{d}{n} + \frac{2 d^2 \log(2/\delta)}{\eps^2 n^2}.
\end{equation}
\end{theorem}

\begin{proof}
Define the mechanism
\begin{equation}
M(X_{1 \cdots n}) = \overline{X} + \cN\left(0, \frac{2 d \log(2/\delta)}{\varepsilon^2 n^2} \cdot \mathbb{I}_{d \times d} \right).
\end{equation}
This mechanism satisfies $(\eps,\delta)$-differential privacy by Lemma~\ref{lem:gauss-mech}, noting that for any pair of adjacent samples $X_{1 \cdots n}$ and $X'_{1 \cdots n}$, $\| \overline{X} - \overline{X}'\|_2^2 \leq \frac{d}{n^2}$.

Let $\sigma^2 = \frac{2 d \log(2/\delta)}{\varepsilon^2 n^2}$.  Note that since the Gaussian noise has mean $0$ and is independent of $\overline{X} - \mu$, we have
\begin{align*}
\ex{}{\| M(X_{1 \cdots n}) - \mu \|_2^2} 
={} &\ex{}{\| \overline{X} - \mu \|_2^2} + \ex{}{\| M(X_{1 \cdots n}) - \overline{X} \|_2^2 } \\ 
\leq{} &\frac{d}{n} +  \ex{}{\| M(X_{1 \cdots n}) - \overline{X} \|_2^2 } \\
={} &\frac{d}{n} + \ex{}{\| \cN(0, \sigma^2 \mathbb{I}_{d \times d}) \|_2^2 } \\
={} &\frac{d}{n} + \sigma^2 d \\
={} &\frac{d}{n} + \frac{2 d^2 \log(2/\delta)}{\eps^2 n^2}.
\end{align*}
\end{proof}

\subsection{Minimax Lower Bounds via Tracing}

\begin{theorem} \label{thm:mean-lb}
For every $n, d \in \N$, $\eps > 0$, and $\delta < 1/96n$, if $\cP$ is the class of all product distributions on $\pmo^{d}$, then for some constant $C > 0$,
\begin{equation*}
\min_{M \in \cM_{\eps,\delta}} \max_{P \in \cP} \ex{X_{1 \cdots n} \sim P,M}{\| M(X_{1 \cdots n}) - \mu \|_2^2} = \Omega\left(\min \left\{ \frac{d^2}{ \eps^2 n^2}, d \right\}\right).
\end{equation*}
\end{theorem}

Note that it is trivial to achieve error $d$ for any distribution using the mechanism $M(X_{1 \cdots n}) \equiv 0$, so the result says that the error must be $\Omega(d^2/\eps^2 n^2)$ whenever this error is significantly smaller than the trivial error of $d$.

\subsubsection{Tracing Attacks}
Before giving the formal proof, we will try to give some intuition for the high-level proof strategy.  The proof can be viewed as constructing a \emph{tracing attack}~\cite{DworkSSU17-arsia} (sometimes called a \emph{membership inference attack}) of the following form. There is an attacker who has the data of some individual $Y$ chosen in one of the two ways: either $Y$ is a random element of the sample $X$, or $Y$ is an independent random sample from the population $P$.  The attacker is given access to the true distribution $P$ and the outcome of the mechanism $M(X)$, and wants to determine which of the two is the case.  If the attacker can succeed, then $M$ cannot be differentially private.  To understand why this is the case, if $Y$ is a member of the dataset, then the attacker should say $Y$ is in the dataset, but if we consider the adjacent dataset $X'$ where we replace $Y$ with some independent sample from $P$, then the attacker will now say $Y$ is independent of the dataset.  Thus, $M(X)$ and $M(X')$ cannot be close in the sense required by differential privacy.

Thus, the proof works by constructing a test statistic $Z = Z(M(X),Y,P),$ that the attacker can use to distinguish the two possibilities for $Y$.  In particular, we show that there is a distribution over populations $P$ such that $\ex{}{Z}$ is small when $Y$ is independent of $X$, but for \emph{every} sufficiently accurate mechanism $M$, $\ex{}{Z}$ is large when $Y$ is a random element of $X$.

\subsubsection{Proof of Theorem~\ref{thm:mean-lb}}

The proof of Theorem~\ref{thm:mean-lb} that we present closely follows the one that appears in Thomas Steinke's Ph.D.~thesis~\cite{Steinke16}.

We start by constructing a ``hard distribution'' over the family of product distributions $\cP$.  Let $\mu = (\mu^1,\dots,\mu^d) \in [-1,1]^d$ consist of $d$ independent draws from the uniform distribution on $[-1,1]$ and let $P_{\mu}$ be the product distribution over $\pmo^{d}$ with mean $\mu$.  Let $X_1,\dots,X_n \sim P_{\mu}$ and $X = (X_1,\dots,X_n)$.  

Let $M \from \pmo^{n \times d} \to [\pm 1]^d$ be any $(\eps,\delta)$-differentially private mechanism and let
\begin{equation}
\alpha^2 = \ex{\mu,X,M}{\| M(X) - \mu\|_2^2 }
\end{equation}
be its expected loss.  We will prove the desired lower bound on $\alpha^2$.

For every element $i$, we define the random variables
\begin{align}
&Z_i = Z_i(M(X),X_i,\mu) = \left\langle M(X) - \mu, X_i - \mu \right\rangle \\
&Z'_{i} = Z'_i(M(X_{\sim i}), X_i, \mu) = \left\langle M(X_{\sim i}) - \mu, X_i - \mu \right\rangle,
\end{align}
where $X_{\sim i}$ denotes $(X_1,\dots,X'_i,\dots,X_n)$ where $X'_i$ is an independent sample from $P_\mu$.  Our goal will be to show that, privacy and accuracy imply both upper and lower bounds on $\ex{}{\sum_i Z_i}$ that depend on $\alpha$, and thereby obtain a bound on $\alpha^2$.

The first claim says that, when $X_i$ is \emph{not} in the sample, then the likelihood random variable has mean $0$ and variance controlled by the expected $\ell_2^2$ error of the mechanism.
\begin{claim} \label{clm:mean-lb-1}
For every $i$, $\ex{}{Z'_i} = 0$, $\var{}{Z'_i} \leq 4\alpha^2$, and $\|Z'_i\|_\infty \leq 4d$.
\end{claim}
\begin{proof}[Proof of Claim~\ref{clm:mean-lb-1}]
Conditioned on any value of $\mu$, $M(X_{\sim i})$ is independent from $X_i$.  Moreover, $\ex{}{X_i - \mu} = 0$, so we have
\begin{align*}
\ex{\mu,X,M}{\langle M(X_{\sim i}) - \mu, X_i - \mu \rangle} &={} \ex{\mu}{\ex{X,M}{\langle M(X_{\sim i}) - \mu, X_i - \mu \rangle}} \\
&={} \ex{\mu}{\left\langle \ex{X,M}{M(X_{\sim i}) - \mu}, \ex{X,M}{X_i - \mu} \right \rangle } \\
&={} \ex{\mu}{\left\langle \ex{X,M}{M(X_{\sim i}) - \mu}, 0 \right \rangle } \\
&= 0.
\end{align*}
For the second part of the claim, since $(X_i - \mu)^2 \leq 4$, we have $\var{}{Z'_i} \leq 4 \cdot \ex{}{\| M(X) - \mu \|_2^2} = 4\alpha^2$.  The final part of the claim follows from the fact that every entry of $M(X_{\sim i}) - \mu$ and $X_i - \mu$ is bounded by $2$ in absolute value, and $Z'_i$ is a sum of $d$ such entries, so its absolute value is always at most $4d$.
\end{proof}

The next claim says that, because $M$ is differentially private, $Z_i$ has similar expectation to $Z'_i$, and thus its expectation is also small.
\begin{claim}\label{clm:mean-lb-2}
$\ex{}{\sum_{i=1}^{n} Z_i} \leq 4n\alpha \eps + 8n  \delta d.$
\end{claim}
\begin{proof}
The proof is a direct calculation using the following inequality, whose proof is relatively simple using the definition of differential privacy:
\begin{equation}
\ex{}{Z_i} \leq \ex{}{Z'_i} + 2\eps  \sqrt{\var{}{Z'_i}} + 2\delta \| Z'_i \|_\infty.
\end{equation}
Given the inequality and Claim~\ref{clm:mean-lb-1}, we have
\begin{equation*}
\ex{}{Z_i} \leq 0 + (2\eps)(2\alpha) + (2\delta)(2d) = 4\eps \alpha + 8 \delta d .
\end{equation*}
The claim now follows by summing over all $i$.
\end{proof}

The final claim says that, because $M$ is accurate, the expected sum of the random variables $Z_i$ is large.
\begin{claim} \label{clm:mean-lb-3}
$\ex{}{\sum_{i=1}^{n} Z_i}  \geq \frac{d}{3} - \alpha^2.$
\end{claim}
The proof relies on the following key lemma, whose proof we omit.
\begin{lemma}[Fingerprinting Lemma~\cite{BunSU17}] \label{lem:fp} If $\mu \in [\pm 1]$ is sampled uniformly, $X_1,\dots,X_n \in \pmo^{n}$ are sampled independently with mean $\mu$, and $f \from \pmo^n \to [\pm 1]$ is any function, then
\begin{equation*}
\ex{\mu,X}{(f(X) - \mu) \cdot \sum_{i=1}^{n} (X_i - \mu)}  \geq \frac{1}{3} - \ex{\mu,X}{(f(X) - \mu)^2}.
\end{equation*}
\end{lemma}
The lemma is somewhat technical, but for intuition, consider the case where $f(X) = \frac{1}{n}\sum_{i} X_i$ is the empirical mean.  In this case we have 
\begin{align*}
  \ex{\mu,X}{(f(X) - \mu) \cdot \sum_{i=1}^n (X_i - \mu)} 
={} \ex{\mu}{\frac{1}{n} \sum_i \ex{X}{ (X_i - \mu)^2} } 
={} \ex{\mu}{\var{}{X_i}} = \frac{1}{3}.
\end{align*}
The lemma says that, when $\mu$ is sampled this way, then any modification of $f$ that reduces the correlation between $f(X)$ and $\sum_i X_i$ will increase the mean-squared-error of $f$ proportionally.  

\begin{proof}[Proof of Claim~\ref{clm:mean-lb-3}]
We can apply the lemma to each coordinate of the estimate $M(X)$.
\begin{align*}
\ex{}{\sum_{i=1}^{n} Z_i} 
={} &\ex{}{\sum_{i=1}^{n} \left\langle M(X) - \mu, X_i - \mu \right\rangle} \\
={} &\sum_{j=1}^{d} \ex{}{(M^j(X) - \mu^j)\cdot \sum_{i=1}^{n} (X_i^j - \mu^j)} \\
\geq{} &\sum_{j=1}^{d} \left( \frac{1}{3} - \ex{}{(M^j(X) - \mu^j)^2} \right) \tag{Lemma~\ref{lem:fp}} \\
={} &\frac{d}{3} - \ex{}{\| M(X) - \mu \|_2^2} 
={} \frac{d}{3} - \alpha^2. \qedhere
\end{align*}
\end{proof}

Combining Claims~\ref{clm:mean-lb-2} and~\ref{clm:mean-lb-3} gives
\begin{equation}
\frac{d}{3} - \alpha^2 \leq 4n\alpha \eps + 8n \delta d.
\end{equation}
Now, if $\alpha^2 \geq \frac{d}{6}$ then we're done, so we'll assume that $\alpha^2 \leq \frac{d}{6}$.  Further, by our assumption on the value of $\delta$, $8n \delta d \leq \frac{d}{12}$.  In this case we can rearrange terms and square both sides to obtain
\begin{equation}
\alpha^2 \geq{} \frac{1}{16 \eps^2 n^2} \left(\frac{d}{3} - \alpha^2 - 8 n\delta d\right)^2 \geq \frac{1}{16 \eps^2 n^2} \left(\frac{d}{12}\right)^2 = \frac{d^2}{2304 \eps^2 n^2}. 
\end{equation}
Combining the two cases for $\alpha^2$ gives $\alpha^2 \geq \min\{ \frac{d}{6}, \frac{d^2}{2304 \eps^2 n^2} \}$, as desired.

\section{CDF Estimation for Discrete, Univariate Distributions}
Suppose we have a distribution $P$ over the ordered, discrete domain $\{1,\dots,D\}$ and let $\cP$ be the family of all such distributions.  The CDF of the distribution is the function $\Phi_{P} : \{1,\dots,D\} \to [0,1]$ given by
\begin{equation}
\Phi_{P}(j) = \pr{}{P \leq j}.
\end{equation}
A natural measure of distance between CDFs is the $\ell_\infty$ distance, as this is the sort of convergence guarantee that the empirical CDF satisfies.  That is, in the non-private setting, the empirical CDF will achieve the minimax rate, which it known by~\cite{DvoretzkyKW56, Massart90} to be
\begin{equation} \label{eq:dkw}
\max_{P \in \cP} \ex{X_{1 \cdots n} \sim P}{\| \Phi_{X} - \Phi_{P} \|_{\infty}} = O\left(\sqrt{\frac{1}{n}} \right).
\end{equation}

\subsection{Private CDF Estimation}

\begin{theorem} \label{thm:cdf-ub}
For every $n \in \N$ and every $\eps,\delta > 0$, there exists an $(\eps,\delta)$-differentially private mechanism $M$ such that
\begin{equation}
\max_{P \in \cP} \ex{X_{1 \cdots n} \sim P}{\| M(X_{1 \cdots n}) - \Phi_{P} \|_{\infty}} = O\left(\sqrt{\frac{1}{n}} + \frac{\log^{3/2}(D) \log^{1/2}(1/\delta)}{\eps n} \right).
\end{equation}
\end{theorem}
\begin{proof}
Assume without loss of generality that $D = 2^{d}$ for an integer $d \geq 1$.  Let $X_{1 \cdots n} \sim P$ be a sample.  By the triangle inequality, we have
\begin{align*}
\ex{X_{1 \cdots n} \sim P}{\| M(X_{1 \cdots n}) - \Phi_{P} \|_{\infty}} &\leq{} \ex{X_{1 \cdots n} \sim P}{\| \Phi_{X} - \Phi_{P} \|_{\infty} + \| M(X_{1 \cdots n}) - \Phi_{X} \|_{\infty}} \\
&\leq{} O(\sqrt{1/n}) + \ex{X_{1 \cdots n} \sim P}{\| M(X_{1 \cdots n}) - \Phi_{X} \|_{\infty}},
\end{align*}
so we will focus on constructing $M$ to approximate $\Phi_{X}$.

For any $\ell = 0,\dots,d-1$ and $j = 1,\dots,2^{d - \ell}$, consider the statistics
\begin{equation}
f_{\ell,j}(X_{1 \cdots n}) = \frac{1}{n} \sum_{i=1}^{n} \mathbbm{1}\{ (j-1)2^{\ell} + 1 \leq X_i \leq j 2^{\ell} \}.
\end{equation}
  Let $f : \{1,\dots,D\}^n \to [0,1]^{2D - 2}$
  be the function whose output consists of all $2D-2$ such counts.  To decipher this notation, for a given $\ell$, the counts $f_{\ell,\cdot}$ form a histogram of $X_{1 \cdots n}$ using consecutive bins of width $2^{\ell}$, and we consider the $\log(D)$ histograms of geometrically increasing width $1,2,4,\dots,D$. 

First, we claim that the function $f$ has low sensitivity---for adjacent samples $X$ and $X'$,
\begin{equation}
\| f(X) - f(X') \|_2^2  \leq \frac{2 \log(D)}{n^2}.
\end{equation}
Thus, we can use the Gaussian mechanism:
\begin{equation}
M'(X_{1 \cdots n}) = f(X_{1 \cdots n}) + \cN\left(0, \frac{2 \log(D) \log(1/\delta)}{\eps^2 n^2} \cdot \mathbb{I}_{2D \times 2D}\right).
\end{equation}
As we will argue, there exists a matrix $A \in \R^{2D \times 2D}$ such that $\Phi_{X} = A \cdot f(X_{1 \cdots n})$.  We will let $M(X_{1 \cdots n}) = A \cdot M'(X_{1 \cdots n})$.  Since differential privacy is closed under post-processing, $M$ inherits the privacy of $M'$.

We will now show how to construct the matrix $A$ and analyze the error of $M$.  For any $j = 1,\dots,D$, we can form the interval $\{1,\dots,j\}$ as the union of at most $\log D$ disjoint intervals of the form we've computed, and therefore we can obtain $\Phi_{X}(j)$ as the sum of at most $\log D$ of the entries of $f(X)$.  For example, if $j = 5$ then we can write
\begin{equation}
\{1,\dots,7\} = \{1,\dots,4\} \cup \{5,6\} \cup \{7\}
\end{equation}
and
\begin{equation}
\Phi_{X}(5) = f_{2,1} + f_{1,3} + f_{0,7}.
\end{equation}
See Figure~\ref{fig:bin-tree-mech} for a visual representation of the decomposition.
Thus we can construct the matrix $A$ using this information.  

\begin{figure}[h]
	\centering
	\includegraphics[width=12cm]{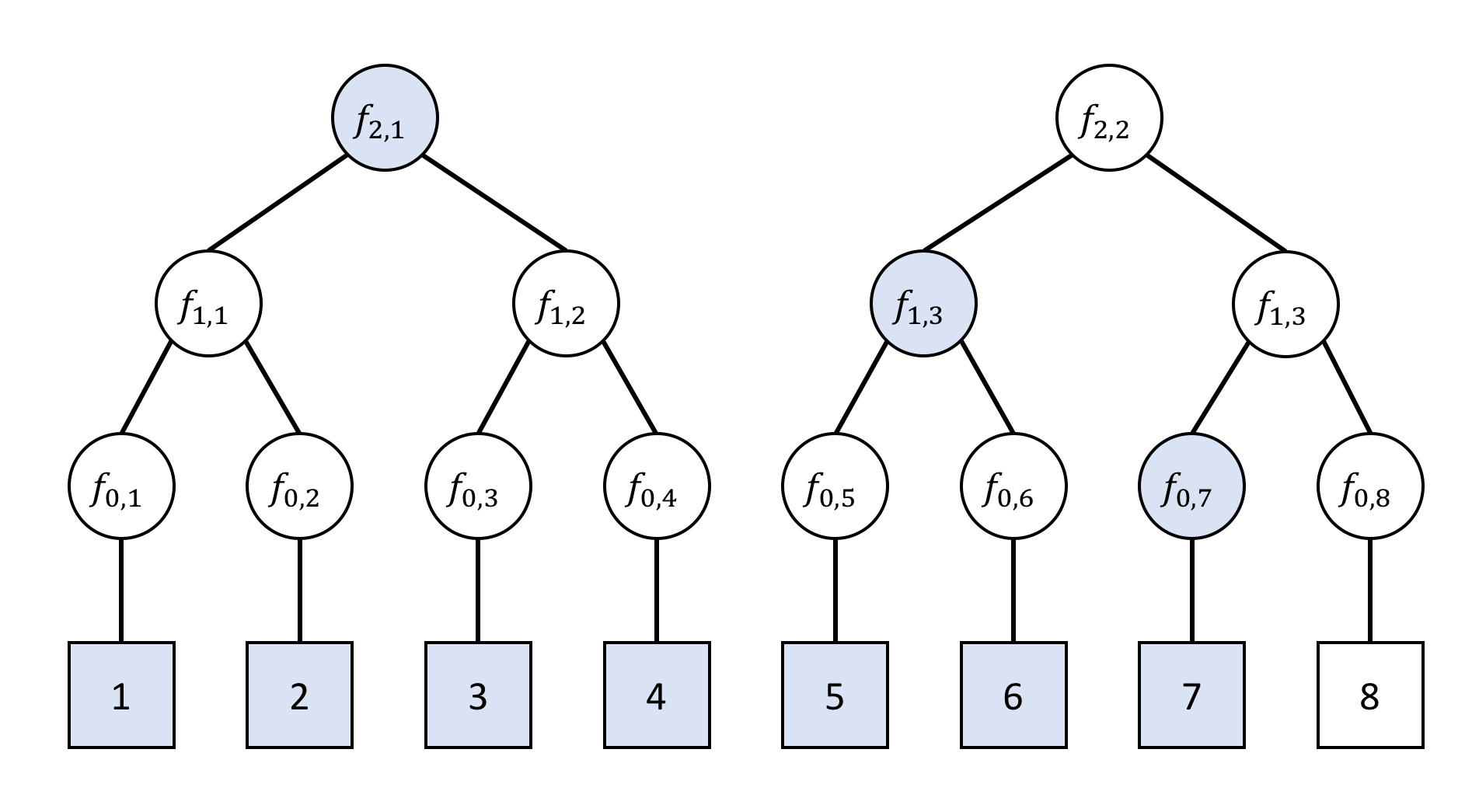}
	\caption{A diagram showing the hierarchical decomposition of the domain $\{1,\dots,8\}$ using 14 intervals.  The highlighted squares represent the interval $\{1,\dots,7\}$ and the highlighted circles show the decomposition of this interval into a union of $3$ intervals in the tree.} \label{fig:bin-tree-mech}
\end{figure}

Note that each entry of $A f(X)$ is the sum of at most $\log(D)$ entries of $f(X)$.  Thus, if we use the output of $M'(X_{1 \cdots n})$ in place of $f(X_{1 \cdots n})$, for every $j$ we obtain
\begin{equation}
\Phi_{X}(j) + \cN(0, \sigma^2) \quad \textrm{for} \quad
\sigma^2 = \frac{ 2 \log^2(D) \log(1/\delta)}{\eps^2 n^2}.
\end{equation}
Applying standard bounds on the expected supremum of a Gaussian process, we have
\begin{equation}
\ex{}{\| M(X_{1 \cdots n}) - \Phi_{X} \|_{\infty}} = O( \sigma \sqrt{\log D}) = O\left(\frac{\log^{3/2}(D) \log^{1/2}(1/\delta)}{\eps n} \right).
\end{equation}
\end{proof}

\subsection{Why Restrict the Domain?}

A drawback of the estimator we constructed is that it only applies to distributions of finite support $\{1,2,\dots,D\}$, albeit with a relatively mild dependence on the support size.  If privacy isn't a concern, then no such restriction is necessary, as the bound~\eqref{eq:dkw} applies equally well to any distribution over $\R$.  Can we construct a differentially private estimator for distributions with infinite support?

Perhaps surprisingly, the answer to this question is no!  Any differentially private estimator for the CDF of the distribution has to have a rate that depends on the support size, and cannot give non-trivial rates for distributions with infinite support.

\begin{theorem}[\cite{BunNSV15}] \label{thm:cdf-lb}
 If $\cP$ consists of all distributions on $\{1,\dots,D\}$, then
\begin{equation}
\min_{M \in \cM_{1, \frac{1}{n}}} \max_{P \in \cP} \ex{X_{1 \cdots n} \sim P}{\| M(X_{1 \cdots n}) - \Phi_{P} \|_{\infty}} = \Omega\left(\frac{\log^* D}{n} \right).\footnote{The notation $\log^* D$ refers to the \href{https://en.wikipedia.org/wiki/Iterated_logarithm}{iterated logarithm}.}
\end{equation}
\end{theorem}

We emphasize that this theorem shouldn't meet with too much alarm, as $\log^* D$ grows remarkably slowly with $D$.  There are differentially private CDF estimators that achieve very mild dependence on $D$~\cite{BeimelNS13b,BunNSV15}, including one nearly matching the lower bound in Theorem~\ref{thm:cdf-lb}.  Moreover, if we want to estimate a distribution over $\R$, and are willing to make some mild regularity conditions on the distribution, then we can approximate it by a distribution with finite support and only increase the rate slightly.  However, what Theorem~\ref{thm:cdf-lb} shows is that there is no ``one-size-fits-all'' solution to private CDF estimation that achieves similar guarantees to the empirical CDF.  That is, the right algorithm has to be tailored somewhat to the application and the assumptions we can make about the distribution. 

\section{More Private Statistics}
Of course, the story doesn't end here!  There's a whole wide world of differentially private statistics beyond what we've mentioned already.  We proceed to survey just a few other directions of study in private statistics. 

\subsection{Parameter and Distribution Estimation}
A number of the early works in differential privacy give methods for differentially private statistical estimation for i.i.d.\ data.  The earliest works~\cite{DinurN03,DworkN04,BlumDMN05,DworkMNS06}, which introduced the Gaussian mechanism, among other foundational results, can be thought of as methods for estimating the mean of a distribution over the hypercube $\{0,1\}^d$ in the $\ell_\infty$ norm.  Tight lower bounds for this problem follow from the tracing attacks introduced in~\cite{BunUV14,SteinkeU15j,DworkSSUV15,BunSU17,SteinkeU17}.  A very recent work of Acharya, Sun, and Zhang~\cite{AcharyaSZ20} adapts classical tools for proving estimation and testing lower bounds (lemmata of Assouad, Fano, and Le Cam) to the differentially private setting. Steinke and Ullman~\cite{SteinkeU17} give tight minimax lower bounds for the weaker guarantee of selecting the largest coordinates of the mean, which were refined by Cai, Wang, and Zhang~\cite{CaiWZ19} to give lower bounds for sparse mean-estimation problems.

Nissim, Raskhodnikova, and Smith introduced the highly general sample-and-aggregate paradigm, which they apply to several learning problems (e.g.,~learning mixtures of Gaussians)~\cite{NissimRS07}.  Later, Smith~\cite{Smith11} showed that this paradigm can be used to transform any estimator for any asymptotically normal, univariate statistic over a bounded data domain into a differentially private one with the same asymptotic convergence rate. 

Subsequent work has focused on both relaxing the assumptions in~\cite{Smith11}, particularly boundedness, and on giving finite-sample guarantees.  Karwa and Vadhan investigated the problem of Gaussian mean estimation, proving the first near-optimal bounds for this setting~\cite{KarwaV18}.  In particular, exploiting concentration properties of Gaussian data allows us to achieve non-trivial results even with unbounded data, which is impossible in general.  Following this, Kamath, Li, Singhal, and Ullman moved to the multivariate setting, investigating the estimation of Gaussians and binary product distributions in total variation distance~\cite{KamathLSU19}.
In certain cases (i.e., Gaussians with identity covariance), this is equivalent to mean estimation in $\ell_2$-distance, though not always.
For example, for binary product distribution, one must estimate the mean in a type of $\chi^2$-distance instead.  The perspective of distribution estimation rather than parameter estimation can be valuable.  Bun, Kamath, Steinke, and Wu~\cite{BunKSW19} develop a primitive for private hypothesis selection, which they apply to learn any coverable class of distributions under pure differential privacy.
Through the lens of distribution estimation, their work implies an upper bound for mean estimation of binary product distributions that bypasses lower bounds for the same problem in the empirical setting.
In addition to work on mean estimation in the sub-Gaussian setting, such as the results discussed earlier, mean estimation has also been studied under weaker moment conditions~\cite{BunS19, KamathSU20}.
Beyond these settings, there has also been study of estimation of discrete multinomials, including estimation in Kolmogorov distance~\cite{BunNSV15} and in total variation distance for structured distributions~\cite{DiakonikolasHS15}, and parameter estimation of Markov Random Fields~\cite{ZhangKKW20}.

A different approach to constructing differentially private estimators is based on robust statistics.  This approah begins with the influential work of Dwork and Lei~\cite{DworkL09}, which introduced the propose-test-release framework, and applied to estimating robust statistics such as the median and interquartile range.  While the definitions in robust statistics and differential privacy are semantically similar, formal connections between the two remain relatively scant, which suggests a productive area for future study.

\subsection{Hypothesis Testing}
An influential work of Homer et al.~\cite{HomerSRDTMPSNC08} demonstrated the vulnerability of classical statistics in a genomic setting, showing that certain $\chi^2$-statistics on many different variables could allow an attacker to determine the presence of an individual in a genome-wide association study (GWAS).
Motivated by these concerns, an early line of work from the statistics community focused on addressing these issues~\cite{VuS09,UhlerSF13,YuFSU14}.

More recently, work on private hypothesis testing can be divided roughly into two lines.  The first focuses on the minimax sample complexity, in a line initiated by Cai, Daskalakis, and Kamath~\cite{CaiDK17}, who give an algorithm for privately testing goodness-of-fit (more precisely, a statistician might refer to this problem as one-sample testing of multinomial data).
A number of subsequent works have essentially settled the complexity of this problem~\cite{AcharyaSZ18, AliakbarpourDR18}, giving tight upper and lower bounds.
Other papers in this line study related problems, including the two-sample version of the problem, independence testing, and goodness-of-fit testing for multivariate product distributions~\cite{AcharyaSZ18, AliakbarpourDR18, AliakbarpourDKR19, CanonneKMUZ19}.
A related paper studies the minimax sample complexity of property \emph{estimation}, rather than testing of discrete distributions, including support size and entropy~\cite{AcharyaKSZ18}.
Other recent works in this vein focus on testing of simple hypotheses~\cite{CummingsKMTZ18, CanonneKMSU19}.  
In particular~\cite{CanonneKMSU19} proves an analogue of the Neyman-Pearson Lemma for differentially private testing of simple hypotheses.  
A paper of Awan and Slavkovic~\cite{AwanS18} gives a universally optimal test when the domain size is two, however Brenner and Nissim~\cite{BrennerN14} shows that such universally optimal tests cannot exist when the domain has more than two elements.
A related problem in this space is private change-point detection~\cite{CummingsKMTZ18, CanonneKMSU19, CummingsKLZ19} -- in this setting, we are given a time series of datapoints which are sampled from a distribution, which at some point, changes to a different distribution.
The goal is to (privately) determine when this point occurs.

Complementary to minimax hypothesis testing, a line of work~\cite{WangLK15,GaboardiLRV16,KiferR17,KakizakiSF17,CampbellBRG18,SwanbergGGRGB19,CouchKSBG19} designs differentially private versions of popular test statistics for testing goodness-of-fit, closeness, and independence, as well as private ANOVA, focusing on the performance at small sample sizes.
Work by Wang et al.~\cite{WangKLK18} focuses on generating statistical approximating distributions for differentially private statistics, which they apply to hypothesis testing problems.

\subsection{Differential Privacy on Graphs}
There is a significant amount of work on differentially private analysis of graphs.  We remark that these algorithms can satisfy either edge or node differential privacy.  The former (easier) guarantee defines a neighboring graph to be one obtained by adding or removing a single edge, while in the latter (harder) setting, a neighboring graph is one that can be obtained by modifying the set of edges connected to a single node.   The main challenge in this area is that most graph statistics can have high sensitivity in the worst-case.

The initial works in this area focused on the empirical setting, and goals range from counting subgraphs~\cite{KarwaRSY11,BlockiBDS13,KasiviswanathanNRS13,ChenZ13, RaskhodnikovaS16} to outputting a privatized graph which approximates the original~\cite{GuptaRU12,BlockiBDS12, Upadhyay13, AroraU19,EliasKKL20}.
In contrast to the setting discussed in most of this series, it seems that there are larger qualitative differences between the study of empirical and population statistics due to the fact that many graph statistics have high worst-case sensitivity, but may have smaller sensitivity on typical graphs from many natural models.

In the population statistics setting, recent work has focused on parameter estimation of the underlying random graph model.
So far this work has given estimators for the $\beta$-model~\cite{KarwaS16} and graphons~\cite{BorgsCS15,BorgsCSZ18a}.
Graphons are a generalization of the stochastic block model, which is, in turn, a generalization of the Erd\H{o}s-R\'enyi model.
Interestingly, the methods of Lipschitz-extensions introduced in the empirical setting by~\cite{BlockiBDS13,KasiviswanathanNRS13} are the main tool used in the statistical setting as well.
While the first works on private graphon estimation were not computationally efficient, a recent focus has been on obviating these issues for certain important cases, such as the Erd\H{o}s-R\'enyi setting~\cite{SealfonU19}.

\section*{Acknowledgments}
Thanks to Adam Smith, who helped kick off this project, and Cl\'ement Canonne, Aaron Roth, and Thomas Steinke for helpful comments.

\bibliographystyle{alpha}
\bibliography{biblio.bib,bibliography/refs.bib}

\newcommand{\etalchar}[1]{$^{#1}$}
\begin{thebibliography}{CKM{\etalchar{+}}19b}

\bibitem[ADKR19]{AliakbarpourDKR19}
Maryam Aliakbarpour, Ilias Diakonikolas, Daniel~M. Kane, and Ronitt Rubinfeld.
\newblock Private testing of distributions via sample permutations.
\newblock In {\em Advances in Neural Information Processing Systems 32},
  NeurIPS '19, pages 10877--10888. Curran Associates, Inc., 2019.

\bibitem[ADR18]{AliakbarpourDR18}
Maryam Aliakbarpour, Ilias Diakonikolas, and Ronitt Rubinfeld.
\newblock Differentially private identity and closeness testing of discrete
  distributions.
\newblock In {\em Proceedings of the 35th International Conference on Machine
  Learning}, ICML '18, pages 169--178. JMLR, Inc., 2018.

\bibitem[AKSZ18]{AcharyaKSZ18}
Jayadev Acharya, Gautam Kamath, Ziteng Sun, and Huanyu Zhang.
\newblock Inspectre: Privately estimating the unseen.
\newblock In {\em Proceedings of the 35th International Conference on Machine
  Learning}, ICML '18, pages 30--39. JMLR, Inc., 2018.

\bibitem[AS18]{AwanS18}
Jordan Awan and Aleksandra Slavkovi{\'c}.
\newblock Differentially private uniformly most powerful tests for binomial
  data.
\newblock In {\em Advances in Neural Information Processing Systems 31},
  NeurIPS '18, pages 4208--4218. Curran Associates, Inc., 2018.

\bibitem[ASZ18]{AcharyaSZ18}
Jayadev Acharya, Ziteng Sun, and Huanyu Zhang.
\newblock Differentially private testing of identity and closeness of discrete
  distributions.
\newblock In {\em Advances in Neural Information Processing Systems 31},
  NeurIPS '18, pages 6878--6891. Curran Associates, Inc., 2018.

\bibitem[ASZ20]{AcharyaSZ20}
Jayadev Acharya, Ziteng Sun, and Huanyu Zhang.
\newblock Differentially private {A}ssouad, {F}ano, and {L}e {C}am.
\newblock {\em arXiv preprint arXiv:2004.06830}, 2020.

\bibitem[AU19]{AroraU19}
Raman Arora and Jalaj Upadhyay.
\newblock On differentially private graph sparsification and applications.
\newblock In {\em Advances in Neural Information Processing Systems 32},
  NeurIPS '19, pages 13378--13389. Curran Associates, Inc., 2019.

\bibitem[BBDS12]{BlockiBDS12}
Jeremiah Blocki, Avrim Blum, Anupam Datta, and Or~Sheffet.
\newblock The {J}ohnson-{L}indenstrauss transform itself preserves differential
  privacy.
\newblock In {\em Proceedings of the 53rd Annual IEEE Symposium on Foundations
  of Computer Science}, FOCS '12, pages 410--419, Washington, DC, USA, 2012.
  IEEE Computer Society.

\bibitem[BBDS13]{BlockiBDS13}
Jeremiah Blocki, Avrim Blum, Anupam Datta, and Or~Sheffet.
\newblock Differentially private data analysis of social networks via
  restricted sensitivity.
\newblock In {\em Proceedings of the 4th Conference on Innovations in
  Theoretical Computer Science}, ITCS '13, pages 87--96, New York, NY, USA,
  2013. ACM.

\bibitem[BCS15]{BorgsCS15}
Christian Borgs, Jennifer Chayes, and Adam Smith.
\newblock Private graphon estimation for sparse graphs.
\newblock In {\em Advances in Neural Information Processing Systems 28}, NIPS
  '15, pages 1369--1377. Curran Associates, Inc., 2015.

\bibitem[BCSZ18]{BorgsCSZ18a}
Christian Borgs, Jennifer Chayes, Adam Smith, and Ilias Zadik.
\newblock Revealing network structure, confidentially: Improved rates for
  node-private graphon estimation.
\newblock In {\em Proceedings of the 59th Annual IEEE Symposium on Foundations
  of Computer Science}, FOCS '18, pages 533--543, Washington, DC, USA, 2018.
  IEEE Computer Society.

\bibitem[BDMN05]{BlumDMN05}
Avrim Blum, Cynthia Dwork, Frank McSherry, and Kobbi Nissim.
\newblock Practical privacy: The {SuLQ} framework.
\newblock In {\em Proceedings of the 24th ACM SIGMOD-SIGACT-SIGART Symposium on
  Principles of Database Systems}, PODS '05, pages 128--138, New York, NY, USA,
  2005. ACM.

\bibitem[BDRS18]{BunDRS18}
Mark Bun, Cynthia Dwork, Guy~N. Rothblum, and Thomas Steinke.
\newblock Composable and versatile privacy via truncated cdp.
\newblock In {\em Proceedings of the 50th Annual ACM Symposium on the Theory of
  Computing}, STOC '18, pages 74--86, New York, NY, USA, 2018. ACM.

\bibitem[BKSW19]{BunKSW19}
Mark Bun, Gautam Kamath, Thomas Steinke, and Zhiwei~Steven Wu.
\newblock Private hypothesis selection.
\newblock In {\em Advances in Neural Information Processing Systems 32},
  NeurIPS '19, pages 156--167. Curran Associates, Inc., 2019.

\bibitem[BN14]{BrennerN14}
Hai Brenner and Kobbi Nissim.
\newblock Impossibility of differentially private universally optimal
  mechanisms.
\newblock {\em SIAM Journal on Computing}, 43(5):1513--1540, 2014.

\bibitem[BNS13]{BeimelNS13b}
Amos Beimel, Kobbi Nissim, and Uri Stemmer.
\newblock Private learning and sanitization: Pure vs. approximate differential
  privacy.
\newblock In {\em Approximation, Randomization, and Combinatorial Optimization.
  Algorithms and Techniques}, RANDOM-APPROX '13, pages 363--378. Springer,
  2013.

\bibitem[BNSV15]{BunNSV15}
Mark Bun, Kobbi Nissim, Uri Stemmer, and Salil Vadhan.
\newblock Differentially private release and learning of threshold functions.
\newblock In {\em Proceedings of the 56th Annual IEEE Symposium on Foundations
  of Computer Science}, FOCS '15, pages 634--649, Washington, DC, USA, 2015.
  IEEE Computer Society.

\bibitem[BS16]{BunS16}
Mark Bun and Thomas Steinke.
\newblock Concentrated differential privacy: Simplifications, extensions, and
  lower bounds.
\newblock In {\em Proceedings of the 14th Conference on Theory of
  Cryptography}, TCC '16-B, pages 635--658, Berlin, Heidelberg, 2016. Springer.

\bibitem[BS19]{BunS19}
Mark Bun and Thomas Steinke.
\newblock Average-case averages: Private algorithms for smooth sensitivity and
  mean estimation.
\newblock In {\em Advances in Neural Information Processing Systems 32},
  NeurIPS '19, pages 181--191. Curran Associates, Inc., 2019.

\bibitem[BSU17]{BunSU17}
Mark Bun, Thomas Steinke, and Jonathan Ullman.
\newblock Make up your mind: The price of online queries in differential
  privacy.
\newblock In {\em Proceedings of the 28th Annual ACM-SIAM Symposium on Discrete
  Algorithms}, SODA '17, pages 1306--1325, Philadelphia, PA, USA, 2017. SIAM.

\bibitem[BUV14]{BunUV14}
Mark Bun, Jonathan Ullman, and Salil Vadhan.
\newblock Fingerprinting codes and the price of approximate differential
  privacy.
\newblock In {\em Proceedings of the 46th Annual ACM Symposium on the Theory of
  Computing}, STOC '14, pages 1--10, New York, NY, USA, 2014. ACM.

\bibitem[CBRG18]{CampbellBRG18}
Zachary Campbell, Andrew Bray, Anna Ritz, and Adam Groce.
\newblock Differentially private {ANOVA} testing.
\newblock In {\em Proceedings of the 2018 International Conference on Data
  Intelligence and Security}, ICDIS '18, pages 281--285, Washington, DC, USA,
  2018. IEEE Computer Society.

\bibitem[CDK17]{CaiDK17}
Bryan Cai, Constantinos Daskalakis, and Gautam Kamath.
\newblock Priv'it: Private and sample efficient identity testing.
\newblock In {\em Proceedings of the 34th International Conference on Machine
  Learning}, ICML '17, pages 635--644. JMLR, Inc., 2017.

\bibitem[CKLZ19]{CummingsKLZ19}
Rachel Cummings, Sara Krehbiel, Yuliia Lut, and Wanrong Zhang.
\newblock Privately detecting changes in unknown distributions.
\newblock {\em arXiv preprint arXiv:1910.01327}, 2019.

\bibitem[CKM{\etalchar{+}}18]{CummingsKMTZ18}
Rachel Cummings, Sara Krehbiel, Yajun Mei, Rui Tuo, and Wanrong Zhang.
\newblock Differentially private change-point detection.
\newblock In {\em Advances in Neural Information Processing Systems 31},
  NeurIPS '18. Curran Associates, Inc., 2018.

\bibitem[CKM{\etalchar{+}}19a]{CanonneKMSU19}
Cl{\'e}ment~L. Canonne, Gautam Kamath, Audra McMillan, Adam Smith, and Jonathan
  Ullman.
\newblock The structure of optimal private tests for simple hypotheses.
\newblock In {\em Proceedings of the 51st Annual ACM Symposium on the Theory of
  Computing}, STOC '19, New York, NY, USA, 2019. ACM.

\bibitem[CKM{\etalchar{+}}19b]{CanonneKMUZ19}
Cl\'ement~L. Canonne, Gautam Kamath, Audra McMillan, Jonathan Ullman, and Lydia
  Zakynthinou.
\newblock Private identity testing for high-dimensional distributions.
\newblock {\em arXiv preprint arXiv:1905.11947}, 2019.

\bibitem[CKS{\etalchar{+}}19]{CouchKSBG19}
Simon Couch, Zeki Kazan, Kaiyan Shi, Andrew Bray, and Adam Groce.
\newblock Differentially private nonparametric hypothesis testing.
\newblock In {\em Proceedings of the 2019 ACM Conference on Computer and
  Communications Security}, CCS '19, New York, NY, USA, 2019. ACM.

\bibitem[CSS11]{ChanSS11}
T-H~Hubert Chan, Elaine Shi, and Dawn Song.
\newblock Private and continual release of statistics.
\newblock {\em ACM Transactions on Information and System Security (TISSEC)},
  14(3):26, 2011.

\bibitem[CWZ19]{CaiWZ19}
T.~Tony Cai, Yichen Wang, and Linjun Zhang.
\newblock The cost of privacy: Optimal rates of convergence for parameter
  estimation with differential privacy.
\newblock {\em arXiv preprint arXiv:1902.04495}, 2019.

\bibitem[CZ13]{ChenZ13}
Shixi Chen and Shuigeng Zhou.
\newblock Recursive mechanism: Towards node differential privacy and
  unrestricted joins.
\newblock In {\em Proceedings of the 2013 ACM SIGMOD International Conference
  on Management of Data}, SIGMOD '13, pages 653--664, New York, NY, USA, 2013.
  ACM.

\bibitem[DHS15]{DiakonikolasHS15}
Ilias Diakonikolas, Moritz Hardt, and Ludwig Schmidt.
\newblock Differentially private learning of structured discrete distributions.
\newblock In {\em Advances in Neural Information Processing Systems 28}, NIPS
  '15, pages 2566--2574. Curran Associates, Inc., 2015.

\bibitem[DKW56]{DvoretzkyKW56}
Aryeh Dvoretzky, Jack Kiefer, and Jacob Wolfowitz.
\newblock Asymptotic minimax character of the sample distribution function and
  of the classical multinomial estimator.
\newblock {\em The Annals of Mathematical Statistics}, 27(3):642--669, 09 1956.

\bibitem[DL09]{DworkL09}
Cynthia Dwork and Jing Lei.
\newblock Differential privacy and robust statistics.
\newblock In {\em Proceedings of the 41st Annual ACM Symposium on the Theory of
  Computing}, STOC '09, pages 371--380, New York, NY, USA, 2009. ACM.

\bibitem[DMNS06]{DworkMNS06}
Cynthia Dwork, Frank McSherry, Kobbi Nissim, and Adam Smith.
\newblock Calibrating noise to sensitivity in private data analysis.
\newblock In {\em Proceedings of the 3rd Conference on Theory of Cryptography},
  TCC '06, pages 265--284, Berlin, Heidelberg, 2006. Springer.

\bibitem[DN03]{DinurN03}
Irit Dinur and Kobbi Nissim.
\newblock Revealing information while preserving privacy.
\newblock In {\em Proceedings of the 22nd ACM SIGMOD-SIGACT-SIGART Symposium on
  Principles of Database Systems}, PODS '03, pages 202--210, New York, NY, USA,
  2003. ACM.

\bibitem[DN04]{DworkN04}
Cynthia Dwork and Kobbi Nissim.
\newblock Privacy-preserving datamining on vertically partitioned databases.
\newblock In {\em Annual International Cryptology Conference}, pages 528--544.
  Springer, 2004.

\bibitem[DNPR10]{DworkNPR10}
Cynthia Dwork, Moni Naor, Toniann Pitassi, and Guy~N. Rothblum.
\newblock Differential privacy under continual observation.
\newblock In {\em Symposium on Theory of Computing (STOC)}, pages 715--724.
  {ACM}, 2010.

\bibitem[DR16]{DworkR16}
Cynthia Dwork and Guy~N. Rothblum.
\newblock Concentrated differential privacy.
\newblock {\em arXiv preprint arXiv:1603.01887}, 2016.

\bibitem[DRS19]{DongRS19}
Jinshuo Dong, Aaron Roth, and Weijie~J. Su.
\newblock Gaussian differential privacy.
\newblock {\em arXiv preprint arXiv:1905.02383}, 2019.

\bibitem[DSS{\etalchar{+}}15]{DworkSSUV15}
Cynthia Dwork, Adam Smith, Thomas Steinke, Jonathan Ullman, and Salil Vadhan.
\newblock Robust traceability from trace amounts.
\newblock In {\em Proceedings of the 56th Annual IEEE Symposium on Foundations
  of Computer Science}, FOCS '15, pages 650--669, Washington, DC, USA, 2015.
  IEEE Computer Society.

\bibitem[DSSU17]{DworkSSU17-arsia}
Cynthia Dwork, Adam Smith, Thomas Steinke, and Jonathan Ullman.
\newblock Exposed! a survey of attacks on private data.
\newblock {\em Annual Review of Statistics and Its Application}, 4:61--84,
  2017.

\bibitem[EKKL20]{EliasKKL20}
Marek Eli{\'a}{\v{s}}, Michael Kapralov, Janardhan Kulkarni, and Yin~Tat Lee.
\newblock Differentially private release of synthetic graphs.
\newblock In {\em Proceedings of the 31st Annual ACM-SIAM Symposium on Discrete
  Algorithms}, SODA '20, pages 560--578, Philadelphia, PA, USA, 2020. SIAM.

\bibitem[FRY10]{FienbergRY10}
Stephen~E. Fienberg, Alessandro Rinaldo, and Xiaolin Yang.
\newblock Differential privacy and the risk-utility tradeoff for
  multi-dimensional contingency tables.
\newblock In {\em Proceedings of the International Conference on Privacy in
  Statistical Databases}, PSD '10, Corfu, Greece, 2010. Springer.

\bibitem[GLRV16]{GaboardiLRV16}
Marco Gaboardi, Hyun{-}Woo Lim, Ryan~M. Rogers, and Salil~P. Vadhan.
\newblock Differentially private chi-squared hypothesis testing: Goodness of
  fit and independence testing.
\newblock In {\em Proceedings of the 33rd International Conference on Machine
  Learning}, ICML '16, pages 1395--1403. JMLR, Inc., 2016.

\bibitem[GRU12]{GuptaRU12}
Anupam Gupta, Aaron Roth, and Jonathan Ullman.
\newblock Iterative constructions and private data release.
\newblock In {\em Proceedings of the 9th Conference on Theory of Cryptography},
  TCC '12, pages 339--356, Berlin, Heidelberg, 2012. Springer.

\bibitem[HSR{\etalchar{+}}08]{HomerSRDTMPSNC08}
Nils Homer, Szabolcs Szelinger, Margot Redman, David Duggan, Waibhav Tembe,
  Jill Muehling, John~V. Pearson, Dietrich~A. Stephan, Stanley~F. Nelson, and
  David~W. Craig.
\newblock Resolving individuals contributing trace amounts of {DNA} to highly
  complex mixtures using high-density {SNP} genotyping microarrays.
\newblock {\em PLoS Genetics}, 4(8):1--9, 2008.

\bibitem[KLSU19]{KamathLSU19}
Gautam Kamath, Jerry Li, Vikrant Singhal, and Jonathan Ullman.
\newblock Privately learning high-dimensional distributions.
\newblock In {\em Proceedings of the 32nd Annual Conference on Learning
  Theory}, COLT '19, pages 1853--1902, 2019.

\bibitem[KNRS13]{KasiviswanathanNRS13}
Shiva~Prasad Kasiviswanathan, Kobbi Nissim, Sofya Raskhodnikova, and Adam
  Smith.
\newblock Analyzing graphs with node differential privacy.
\newblock In {\em Proceedings of the 10th Conference on Theory of
  Cryptography}, TCC '13, pages 457--476, Berlin, Heidelberg, 2013. Springer.

\bibitem[KR17]{KiferR17}
Daniel Kifer and Ryan~M. Rogers.
\newblock A new class of private chi-square tests.
\newblock In {\em Proceedings of the 20th International Conference on
  Artificial Intelligence and Statistics}, AISTATS '17, pages 991--1000. JMLR,
  Inc., 2017.

\bibitem[KRSY11]{KarwaRSY11}
Vishesh Karwa, Sofya Raskhodnikova, Adam Smith, and Grigory Yaroslavtsev.
\newblock Private analysis of graph structure.
\newblock {\em Proceedings of the VLDB Endowment}, 4(11):1146--1157, 2011.

\bibitem[KS16]{KarwaS16}
Vishesh Karwa and Aleksandra Slavkovi{\'c}.
\newblock Inference using noisy degrees: Differentially private $\beta$-model
  and synthetic graphs.
\newblock {\em The Annals of Statistics}, 44(1):87--112, 2016.

\bibitem[KSF17]{KakizakiSF17}
Kazuya Kakizaki, Jun Sakuma, and Kazuto Fukuchi.
\newblock Differentially private chi-squared test by unit circle mechanism.
\newblock In {\em Proceedings of the 34th International Conference on Machine
  Learning}, ICML '17, pages 1761--1770. JMLR, Inc., 2017.

\bibitem[KSU20]{KamathSU20}
Gautam Kamath, Vikrant Singhal, and Jonathan Ullman.
\newblock Private mean estimation of heavy-tailed distributions.
\newblock {\em arXiv preprint arXiv:2002.09464}, 2020.

\bibitem[KV18]{KarwaV18}
Vishesh Karwa and Salil Vadhan.
\newblock Finite sample differentially private confidence intervals.
\newblock In {\em Proceedings of the 9th Conference on Innovations in
  Theoretical Computer Science}, ITCS '18, pages 44:1--44:9, Dagstuhl, Germany,
  2018. Schloss Dagstuhl--Leibniz-Zentrum fuer Informatik.

\bibitem[Mas90]{Massart90}
P.~Massart.
\newblock The tight constant in the {D}voretzky-{K}iefer-{W}olfowitz
  inequality.
\newblock {\em The Annals of Probability}, 18(3):1269--1283, 07 1990.

\bibitem[Mir17]{Mironov17}
Ilya Mironov.
\newblock R{\'e}nyi differential privacy.
\newblock In {\em Proceedings of the 30th IEEE Computer Security Foundations
  Symposium}, CSF '17, pages 263--275, Washington, DC, USA, 2017. IEEE Computer
  Society.

\bibitem[NRS07]{NissimRS07}
Kobbi Nissim, Sofya Raskhodnikova, and Adam Smith.
\newblock Smooth sensitivity and sampling in private data analysis.
\newblock In {\em Proceedings of the 39th Annual ACM Symposium on the Theory of
  Computing}, STOC '07, pages 75--84, New York, NY, USA, 2007. ACM.

\bibitem[RS16]{RaskhodnikovaS16}
Sofya Raskhodnikova and Adam~D. Smith.
\newblock Lipschitz extensions for node-private graph statistics and the
  generalized exponential mechanism.
\newblock In {\em Proceedings of the 57th Annual IEEE Symposium on Foundations
  of Computer Science}, FOCS '16, pages 495--504, Washington, DC, USA, 2016.
  IEEE Computer Society.

\bibitem[SGHG{\etalchar{+}}19]{SwanbergGGRGB19}
Marika Swanberg, Ira Globus-Harris, Iris Griffith, Anna Ritz, Adam Groce, and
  Andrew Bray.
\newblock Improved differentially private analysis of variance.
\newblock {\em Proceedings on Privacy Enhancing Technologies}, 2019(3), 2019.

\bibitem[Smi11]{Smith11}
Adam Smith.
\newblock Privacy-preserving statistical estimation with optimal convergence
  rates.
\newblock In {\em Proceedings of the 43rd Annual ACM Symposium on the Theory of
  Computing}, STOC '11, pages 813--822, New York, NY, USA, 2011. ACM.

\bibitem[Ste16]{Steinke16}
Thomas~Alexander Steinke.
\newblock {\em Upper and Lower Bounds for Privacy and Adaptivity in Algorithmic
  Data Analysis}.
\newblock PhD thesis, 2016.

\bibitem[SU17a]{SteinkeU15j}
Thomas Steinke and Jonathan Ullman.
\newblock Between pure and approximate differential privacy.
\newblock {\em Journal of Privacy and Confidentiality}, 7(2), 2017.

\bibitem[SU17b]{SteinkeU17}
Thomas Steinke and Jonathan Ullman.
\newblock Tight lower bounds for differentially private selection.
\newblock In {\em 58th Annual IEEE Symposium on Foundations of Computer
  Science}, FOCS '17, pages 552--563, Berkeley, CA, 2017.

\bibitem[SU19]{SealfonU19}
Adam Sealfon and Jonathan Ullman.
\newblock Efficiently estimating {E}rdos-{R}enyi graphs with node differential
  privacy.
\newblock In {\em Advances in Neural Information Processing Systems 32},
  NeurIPS '19, pages 3765--3775. Curran Associates, Inc., 2019.

\bibitem[Upa13]{Upadhyay13}
Jalaj Upadhyay.
\newblock Random projections, graph sparsification, and differential privacy.
\newblock In {\em Proceedings of the 19th Annual International Conference on
  the Theory and Application of Cryptology and Information Security}, ASIACRYPT
  '13, pages 276--295, Berlin, Heidelberg, 2013. Springer.

\bibitem[USF13]{UhlerSF13}
Caroline Uhler, Aleksandra Slavkovi{\'c}, and Stephen~E. Fienberg.
\newblock Privacy-preserving data sharing for genome-wide association studies.
\newblock {\em The Journal of Privacy and Confidentiality}, 5(1):137--166,
  2013.

\bibitem[VS09]{VuS09}
Duy Vu and Aleksandra Slavkovi{\'c}.
\newblock Differential privacy for clinical trial data: Preliminary
  evaluations.
\newblock In {\em 2009 IEEE International Conference on Data Mining Workshops},
  ICDMW '09, pages 138--143. IEEE, 2009.

\bibitem[WKLK18]{WangKLK18}
Yue Wang, Daniel Kifer, Jaewoo Lee, and Vishesh Karwa.
\newblock Statistical approximating distributions under differential privacy.
\newblock {\em The Journal of Privacy and Confidentiality}, 8(1):1--33, 2018.

\bibitem[WLK15]{WangLK15}
Yue Wang, Jaewoo Lee, and Daniel Kifer.
\newblock Revisiting differentially private hypothesis tests for categorical
  data.
\newblock {\em arXiv preprint arXiv:1511.03376}, 2015.

\bibitem[YFSU14]{YuFSU14}
Fei Yu, Stephen~E. Fienberg, Aleksandra~B. Slavkovi{\'c}, and Caroline Uhler.
\newblock Scalable privacy-preserving data sharing methodology for genome-wide
  association studies.
\newblock {\em Journal of Biomedical Informatics}, 50:133--141, 2014.

\bibitem[ZKKW20]{ZhangKKW20}
Huanyu Zhang, Gautam Kamath, Janardhan Kulkarni, and Zhiwei~Steven Wu.
\newblock Privately learning {M}arkov random fields.
\newblock {\em arXiv preprint arXiv:2002.09463}, 2020.

\end{thebibliography}

\end{document}